\documentclass{article}

\newcommand{\shortversion}[1]{}
\newcommand{\longversion}[1]{#1}

\shortversion{\usepackage{ijcai11}}

\longversion{\usepackage[hscale=0.7,vscale=0.75]{geometry}}

\usepackage{tikz}

\shortversion{\usepackage{times}}
\usepackage[T1]{fontenc}
\usepackage{textcomp}

\usepackage{enumitem}

\usepackage{amsthm}
\usepackage{amsfonts}
\usepackage{amsmath}
\usepackage{url}
\usepackage{caption} 

\shortversion{\captionsetup{font=small,figurename=Fig.,aboveskip=0pt,belowskip=-3pt,position=bottom}}

\shortversion{\leftmargin 16pt \leftmargini\leftmargin \leftmarginii 16pt}

\newcommand{\citex}[1]{\citeauthor{#1}~\shortcite{#1}}

\newtheorem{lemma}{Lemma} 
\newtheorem{theorem}{Theorem} 
\newtheorem{corollary}{Corollary}

\newcommand{\SB}{\{\,}%
\newcommand{\SM}{\;{:}\;}%
\newcommand{\SE}{\,\}}%

\newcommand{\Card}[1]{|#1|}

\let\phi=\varphi
\let\epsilon=\varepsilon

\newcommand{\UP}{\text{\normalfont{UP}}}
\newcommand{\Nat}{\mathbb{N}}

\newcommand{\scope}{\mathit{scope}}
\newcommand{\dom}{\mathit{dom}}

\newcommand{\NP}{\text{\normalfont NP}}
\newcommand{\coNP}{\text{\normalfont coNP}}
\newcommand{\FPT}{\text{\normalfont FPT}}

\makeatletter  

\longversion{
  \def\leftcite{\@up[}\def\rightcite{\@up]}
  
  \def\cite{\def\citeauthoryear##1##2{\def\@thisauthor{##1}%
               \ifx \@lastauthor \@thisauthor \relax \else##1, \fi ##2}\@icite}
  \def\shortcite{\def\citeauthoryear##1##2{##2}\@icite}
  
  \def\citeauthor{\def\citeauthoryear##1##2{##1}\@nbcite}
  \def\citeyear{\def\citeauthoryear##1##2{##2}\@nbcite}
  
  \def\@icite{\leavevmode\def\@citeseppen{-1000}%
   \def\@cite##1##2{\leftcite\nobreak\hskip 0in{##1\if@tempswa , ##2\fi}\rightcite}%
   \@ifnextchar [{\@tempswatrue\@citex}{\@tempswafalse\@citex[]}}
  \def\@nbcite{\leavevmode\def\@citeseppen{1000}%
   \def\@cite##1##2{{##1\if@tempswa , ##2\fi}}%
   \@ifnextchar [{\@tempswatrue\@citex}{\@tempswafalse\@citex[]}}
  
  \def\@citex[#1]#2{%
    \def\@lastauthor{}\def\@citea{}%
    \@cite{\@for\@citeb:=#2\do
      {\@citea\def\@citea{;\penalty\@citeseppen\ }%
       \if@filesw\immediate\write\@auxout{\string\citation{\@citeb}}\fi
       \@ifundefined{b@\@citeb}{\def\@thisauthor{}{\bf ?}\@warning
         {Citation `\@citeb' on page \thepage \space undefined}}%
       {\csname b@\@citeb\endcsname}\let\@lastauthor\@thisauthor}}{#1}}
  
  \def\@biblabel#1{\def\citeauthoryear##1##2{##1, ##2}\@up{[}#1\@up{]}\hfill}
  
  \def\@up#1{\leavevmode\raise.2ex\hbox{#1}}
}

\shortversion{
 \def\thm@space@setup{%
   \thm@preskip=4pt \thm@postskip=\thm@preskip}
 \renewenvironment{proof}[1][\proofname]{%
   \par  \setlength{\topsep}{0pt}  \pushQED{\qed}%
   \normalfont 
   \trivlist  \item[\hskip\labelsep        \itshape
   #1\@addpunct{.}]\ignorespaces}{%
   \popQED\endtrivlist\@endpefalse\vspace{2pt}}

 \def\section{\@startsection{section}{1}{\z@}{-9pt plus
     -3pt minus -2pt}{3pt plus 2pt minus 1pt}{\Large\bf\raggedright}}
 \def\subsection{\@startsection{subsection}{2}{\z@}{-7pt plus
     -2pt minus -2pt}{2pt plus 2pt minus 1pt}{\large\bf\raggedright}}
 \def\subsubsection{\@startsection{subparagraph}{3}{\z@}{-5pt plus
    -2pt minus -1pt}{1pt plus 1pt minus 1pt}{\normalsize\bf\raggedright}}
 \def\paragraph{\@startsection{paragraph}{4}{\z@}%
                                     {1pt \@plus1ex \@minus.2ex}%
                                     {-1em}%
                                     {\normalfont\normalsize\bfseries}}
}

\makeatother   

\usepackage{xspace}
\usepackage{boxedminipage}
\newtheorem{boldclaim}{Claim}

\newtheorem{property}{Property}
\newcommand{\set}[1]{\{ #1 \}}

\newcommand{\NV}{\textsc{NValue}\xspace}
\newcommand{\AMNV}{\textsc{AtMost-NValue}\xspace}
\newcommand{\ALNV}{\textsc{AtLeast-NValue}\xspace}
\newcommand{\EGClong}{\textsc{Extended Global Cardinality}\xspace}
\newcommand{\EGC}{\textsc{EGC}\xspace}
\newcommand{\AD}{\textsc{AllDifferent}\xspace}
\newcommand{\nbholes}{\text{\#holes}}
\newcommand{\fpt}{fixed-parameter tractable\xspace}

\newcommand{\mydef}[1]{\emph{#1}} 
\newcommand{\mydefalt}[2]{\emph{#1}} 
\shortversion{\newcommand{\myIff}{iff\xspace}}
\longversion{\newcommand{\myIff}{if and only if\xspace}}

\newcommand{\Bessiere}{Bessi{\`e}re\xspace}
\newcommand{\hac}{hyper arc consistent\xspace}
\newcommand{\HAC}{HAC\xspace}

\newcommand{\lep}{\mathsf{l}}
\newcommand{\rep}{\mathsf{r}}
\newcommand{\ivl}{\mathsf{ivl}}

\newcommand{\RedIncl}{\textbf{Red-$\subseteq$}\xspace}
\newcommand{\RedDom}{\textbf{Red-Dom}\xspace}
\newcommand{\RedUnit}{\textbf{Red-Unit}\xspace}

\shortversion{
 \renewenvironment{abstract}%
 {\centerline{\Large\bf Abstract}
 \vspace{0.5ex}
 \begin{list}{}%
          {\setlength{\leftmargin}{11pt}%
          \setlength{\rightmargin}{11pt}}%
          \item[]%
 }
 {\end{list}}
}

\title{Kernels for Global Constraints\thanks{Research funded by the ERC (COMPLEX REASON, 239962).}}

\shortversion{
 \author{Serge Gaspers \and Stefan Szeider\\
 Vienna University of Technology\\
 Vienna, Austria\\
 gaspers@kr.tuwien.ac.at, stefan@szeider.net}
}
\longversion{
 \author{Serge Gaspers \and Stefan Szeider}
 \date{\normalsize Vienna University of Technology\\
 Vienna, Austria\\
 gaspers@kr.tuwien.ac.at, stefan@szeider.net}
}

\begin{document}

  \maketitle

  \begin{abstract}
    \Bessiere \emph{et al.}~(AAAI'08) showed that several intractable
    global constraints can be efficiently propagated when certain natural
    problem parameters are small. In particular, 
    the complete
    propagation of a global constraint is fixed-parameter tractable in $k$ --
    the number of holes in domains -- whenever bound consistency can be enforced
    in polynomial time; this applies 
    to the global constraints \AMNV and \EGClong (\EGC).

   \smallskip

    In this paper we extend this line of research and introduce the
    concept of \emph{reduction to a problem kernel}, a key concept of
    parameterized complexity, to the field of global constraints. In
    particular, we show that the consistency problem for \AMNV constraints admits a
    linear time reduction to an equivalent instance on $O(k^2)$ variables and domain values. 
    This small kernel can be used to speed up the complete
    propagation of \NV constraints. We contrast this result by
    showing that the consistency problem for \EGC constraints does not admit a reduction to a
    polynomial problem kernel unless the polynomial hierarchy collapses.

  \end{abstract}

\section{Introduction}

Constraint programming (CP) offers a powerful framework for efficient
modeling and solving of a wide range of hard problems
\cite{RossiVanBeekWalsh06}. At the heart of efficient CP solvers are
so-called \emph{global constraints} that specify patterns that
frequently occur in real-world problems. Efficient propagation
algorithms for global constraints help speed up the solver
significantly~\cite{HoeveKatriel06}. For instance, a frequently
occurring pattern is that we require that certain variables must all
take different values (e.g., activities requiring the same resource must
all be assigned different times).  Therefore most constraint solvers
provide a global \AD constraint and algorithms for its
propagation. Unfortunately, for several important global
constraints a complete propagation is NP-hard, and one switches
therefore to incomplete propagation such as bound
consistency~\cite{BessiereEtAl04}. In their AAAI'08 paper,
\citex{BessiereEtal08} showed that a complete propagation of several intractable constraints can
efficiently be done as long as certain natural problem
parameters are small, i.e., the propagation is
\emph{fixed-parameter tractable} \cite{DowneyFellows99}.  Among others,
they showed fixed-parameter tractability of the \ALNV
and \EGClong (\EGC) constraints parameterized
by the number of ``holes'' in the domains of the variables. If there are no holes, then all
domains are intervals and complete propagation is polynomial by classical
results; thus the number of holes provides a way of \emph{scaling up} the
nice properties of constraints with interval domains.


In this paper we bring this approach a significant step forward, picking
up a long-term research objective suggested by \citex{BessiereEtal08}\longversion{ in
their concluding remarks}: whether intractable global constraints admit a
\emph{reduction to a problem kernel} or \emph{kernelization}.

Kernelization is an important algorithmic technique that has become the
subject of a very active field in state-of-the-art combinatorial
optimization~(see, e.g., \longversion{the references in }%
\cite{Fellows06,GuoNiedermeier07,Rosamond10}). \longversion{Kernelization}\shortversion{It} can be
seen as a \emph{preprocessing with performance guarantee} that reduces
\longversion{a problem}\shortversion{an} instance in polynomial time to an equivalent instance, the
\emph{kernel}, whose size is a function of the parameter
\cite{Fellows06,GuoNiedermeier07,Fomin10}.

Once a kernel is obtained, the time required to solve the instance is
a function of the parameter only and therefore independent of the input
size. Consequently one aims at kernels that are as small as possible;
the kernel size provides a performance guarantee for the preprocessing.
Some NP-hard\longversion{ combinatorial} problems such as $k$-\textsc{Vertex Cover}
admit polynomially sized kernels, for others such as $k$-\textsc{Path} an
exponential kernel is the best one can hope for
\cite{BodlaenderDowneyFellowsHermelin09}.

Kernelization fits
perfectly into the context of CP where preprocessing and data
reduction (e.g., in terms of local consistency algorithms, propagation,
and domain filtering) are key methods \cite{Bessiere06,HoeveKatriel06}.

\paragraph{Results} 
Do the global constraints \AMNV and \EGC admit
polynomial kernels?  We show that the answer is \emph{``yes''} for the
former and \emph{``no''} for the latter. 

More specifically, we present a \emph{linear time} preprocessing
algorithm that reduces an \AMNV constraint $C$ with
$k$ holes to a consistency-equivalent \AMNV constraint
$C'$ of size polynomial in $k$. In fact, 
$C'$ has at most $O(k^2)$ variables and $O(k^2)$ domain
values.
We also give an improved branching algorithm checking the consistency of $C'$
in time $O(1.6181^k)$.
The combination of kernelization and branching yields efficient algorithms for the consistency and
propagation of (\textsc{AtMost}-)\textsc{NValue} constraints.

On the other hand, we show that a similar
result is unlikely for the \EGC constraint: One cannot reduce an \EGC
constraint $C$ with $k$ holes in polynomial time to a
consistency-equivalent \EGC constraint
$C'$ of size polynomial in $k$.  This result is subject
to the complexity theoretic assumption that $\NP \not \subseteq \coNP/\text{\normalfont poly}$ whose failure implies the collapse
of the Polynomial Hierarchy to its third level, which is
considered highly unlikely by complexity theorists.

\section{Formal Background}

\paragraph{Parameterized Complexity}

\newcommand{\PP}{P}
\newcommand{\QQ}{Q}

A \emph{parameterized problem} $\PP$ is a subset of $\Sigma^* \times \Nat$
for some finite alphabet $\Sigma$. For a problem instance $(x,k) \in
\Sigma^* \times \Nat$ we call $x$ the main part and $k$ the parameter.
A parameterized problem $\PP$ is \emph{\fpt} (\FPT) if a
given instance $(x, k)$ can be solved in time $O(f(k) \cdot p(\Card{x}))$
where $f$ is an arbitrary computable function of $k$ and $p$ is a
polynomial in the input size $\Card{x}$.

\paragraph{Kernels} A \emph{kernelization} for a parameterized problem $\PP
\subseteq \Sigma^* \times \Nat$ is an algorithm that, given $(x, k) \in
\Sigma^* \times \Nat$, outputs in time polynomial in $\Card{x}+k$ a pair
$(x', k') \in \Sigma^* \times \Nat$ such that (i)~$(x,k) \in \PP$ \myIff 
$(x',k') \in \PP$ and (ii)~$\Card{x'}+k′'\leq g(k)$, where $g$ is
an arbitrary computable function. The function $g$ is referred to as the
\emph{size} of the kernel. If $g$ is a polynomial then we say that $\PP$
admits a \emph{polynomial kernel}.




\paragraph{Global Constraints}

An instance of the constraint satisfaction problem (CSP)
consists of a
set of variables, each with a finite domain of values, and a set of
constraints specifying allowed combinations of values for some subset of
variables. We denote by $\dom(x)$ the domain of a variable $x$ and by
$\scope(C)$ the subset of variables involved in a constraint $C$.  An
\emph{instantiation} is an assignment $\alpha$ of values to variables
such that $\alpha(x)\in \dom(x)$ for each variable $x \in \scope(C)$.  A
constraint can be specified extensionally by listing all legal
instantiations of its variables or intensionally, by giving an
expression involving the variables in the constraint scope
\cite{Smith06}. \emph{Global constraints} are certain extensionally
described constraints involving an arbitrary number of variables
\cite{HoeveKatriel06}.  For example, an instantiation is legal for an
\AD global constraint $C$ if 
it assigns pairwise different values to the variables in $\scope(C)$.

\paragraph{Consistency} 
A global constraint $C$ is \emph{consistent} if there is a legal
instantiation of its variables. \longversion{The constraint }$C$ is \emph{\hac} (\emph{\HAC{}}) if
for each variable $x\in \scope(C)$ and each value $v\in \dom(x)$, there
is a legal instantiation $\alpha$ such that $\alpha(x)=v$ (in that
case we say that $C$ supports $v$ for $x$). In the literature, \HAC is also called
\emph{domain consistent} or \emph{generalized arc consistent}.
\longversion{The constraint }$C$ is \emph{bound consistent} if when a variable $x\in \scope(C)$
is assigned the minimum or maximum value of its domain, there are compatible values
between the minimum and maximum domain value for all other variables in $\scope(C)$.
The main algorithmic problems
for a global constraint $C$ are the following: 
\emph{Consistency}, to decide whether $C$ is consistent, and
\emph{Enforcing \HAC{}}, to remove from all domains the values that are not
supported by the respective variable.

It is clear that if \HAC can be enforced in polynomial time for a constraint
$C$, then the consistency of $C$ can also be decided in polynomial time
(we just need to see if any domain became empty).  The reverse is true
for constraints that satisfy a certain closure property (see
\cite{HoeveKatriel06}), which is the case for most constraints of
practical use, and in particular for all constraints considered below. The
same correspondence holds with respect to fixed-parameter
tractability. Hence, we will focus mainly on Consistency.

\section{NValue Constraints}

The \NV constraint was introduced by \citex{PachetRoy99}.
For a set of variables
$X$ and a variable $N$,
$\NV(X,N)$ is
consistent if there is an assignment $\alpha$
such that exactly $\alpha(N)$ different values are used for the variables in $X$.
\AD is the special case where $\dom(N)=\set{|X|}$.
\citex{Beldiceanu01} and \citex{BessiereEtal06} decompose \NV constraints into two other global constraints:
\AMNV and \ALNV, which
require that 
at most $N$ or at least $N$ values are used for the variables in $X$,
respectively.
The Consistency problem is $\NP$-complete for \NV and
\AMNV constraints, and polynomial time solvable for 
\ALNV constraints.

\smallskip

For checking the consistency of an \AMNV constraint $C$, we are given an instance $\mathcal{I}$ consisting of a set of variables $X=\set{x_1, \ldots, x_n}$,
a totally ordered set of values $D$, a map $\dom: X \rightarrow 2^D$ assigning a non-empty domain $\dom(x)\subseteq D$ to each variable
$x\in X$, and an integer $N$.\footnote{If $D$ is not part of the input (or is very large), we may construct $D$ by sorting the set of all endpoints of intervals in time $O((n+k) \log(n+k))$. Since, w.l.o.g., a solution
contains only endpoints of intervals, this step does not compromise the correctness.}
A \mydef{hole} in a subset $D'\subseteq D$ is a couple $(u,w) \in D' \times D'$, such that there is a $v \in D \setminus D'$ with $u<v<w$ and there is no $v'\in D'$ with $u<v'<w$. We denote the number of holes in the domain of a variable $x\in X$ by $\nbholes(x)$. 
The parameter of the consistency problem for \AMNV constraints is $k= \sum_{x\in X} \nbholes(x)$. 
%
%
%
An \emph{interval} $I=[v_1,v_2]$ of a variable $x$ is an inclusion-wise maximal hole-free subset of its domain. Its \mydefalt{left endpoint}{left endpoint} $\lep(I)$ and \mydefalt{right endpoint}{right endpoint} $\rep(I)$ are the values $v_1$ and $v_2$, respectively.
Fig.~\ref{fig:input} gives an example of an instance and its interval representation. We assume that instances are given by a succinct description, in which the domain of a variable is given by the left and right endpoint of each of its intervals. As the number of intervals of the instance $\mathcal{I}=(X,D,\dom,N)$ is $n+k$, its size is $|\mathcal{I}|=O(n+|D|+k)$. In case $\dom$ is given by an extensive list of the values in the domain of each variable, a succinct representation can be computed in linear time.

A greedy algorithm by \citex{Beldiceanu01} checks the consistency of an \AMNV constraint in linear time when all domains are intervals (i.e., $k=0$).
Further,  \citex{BessiereEtal08} have shown that Consistency (and Enforcing \HAC) is \FPT, parameterized by the number of holes, for all constraints for which bound consistency can be enforced in polynomial time.
A simple algorithm for checking the consistency of \AMNV goes over all instances obtained from restricting the domain of each variable to one of its intervals, and executes the algorithm of \cite{Beldiceanu01} for each of these $2^k$ instances. The running time of this \FPT\ algorithm is clearly bounded by $O(2^k \cdot |\mathcal{I}|)$.

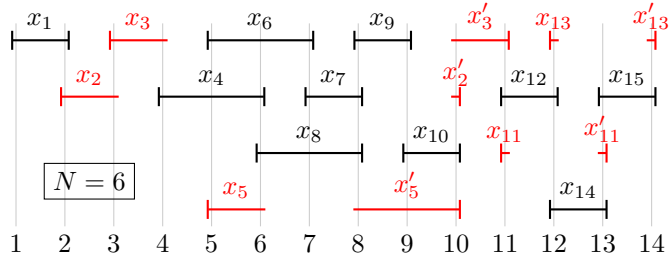
\begin{figure} \centering
 \shortversion{\begin{tikzpicture}[xscale=0.6,yscale=0.65]}
 \longversion{\begin{tikzpicture}[xscale=0.65,yscale=0.75]}
  \tikzset{ivl/.style={thick},
  var/.style={midway,above=0.5pt}}
  \shortversion{\tikzset{ivlopt/.style={thick,black!60}}}
  \longversion{\tikzset{ivlopt/.style={thick,red}}}
 
  \foreach \x in {1,...,14} \draw (\x,-0.6) node {$\x$};
  \foreach \x in {1,...,14} \draw[thin,black!20] (\x,-0.3) -- (\x,3.3);
  
  \draw (2.5,0.5) node[rectangle,draw] {$N=6$};
  
  \draw[ivl,|-|] (0.9,3)-- node[var] {$x_1$} (2.1,3);
  \draw[ivlopt,|-] (2.9,3)-- node[var] {$x_3$} (4.1,3);
  \draw[ivl,|-|] (4.9,3)-- node[var] {$x_6$} (7.1,3);
  \draw[ivl,|-|] (7.9,3)-- node[var] {$x_9$} (9.1,3);
  \draw[ivlopt,-|] (9.9,3)-- node[var] {$x_3'$} (11.1,3);
  \draw[ivlopt,|-] (11.9,3)-- node[var] {$x_{13}$} (12.1,3);
  \draw[ivlopt,-|] (13.9,3)-- node[var] {$x_{13}'$} (14.1,3);
  \draw[ivlopt,|-] (1.9,2)-- node[var] {$x_2$} (3.1,2);
  \draw[ivl,|-|] (3.9,2)-- node[var] {$x_4$} (6.1,2);
  \draw[ivl,|-|] (6.9,2)-- node[var] {$x_7$} (8.1,2);
  \draw[ivlopt,-|] (9.9,2)-- node[var] {$x_2'$} (10.1,2);
  \draw[ivl,|-|] (10.9,2)-- node[var] {$x_{12}$} (12.1,2);
  \draw[ivl,|-|] (12.9,2)-- node[var] {$x_{15}$} (14.1,2);
  \draw[ivl,|-|] (5.9,1)-- node[var] {$x_8$} (8.1,1);
  \draw[ivl,|-|] (8.9,1)-- node[var] {$x_{10}$} (10.1,1);
  \draw[ivlopt,|-] (10.9,1)-- node[var] {$x_{11}$} (11.1,1);
  \draw[ivlopt,-|] (12.9,1)-- node[var] {$x_{11}'$} (13.1,1);
  \draw[ivlopt,|-] (4.9,0)-- node[var] {$x_5$} (6.1,0);
  \draw[ivlopt,-|] (7.9,0)-- node[var] {$x_5'$} (10.1,0);
  \draw[ivl,|-|] (11.9,0)-- node[var] {$x_{14}$} (13.1,0);
 \end{tikzpicture}
 \caption{\label{fig:input} Interval representation of an \AMNV instance $\mathcal{I}=(X,D,\dom,N)$, with $X=\set{x_1,\ldots, x_{15}}$, $N=6$, $D=\set{1,\ldots,14}$, and $\dom(x_1)=\set{1,2}, \dom(x_2)=\set{2,3,10}$, etc.}
\end{figure}

\longversion{
In the realm of parameterized complexity it is then natural to ask whether \AMNV has a polynomial kernel.
 In the next subsection, we give a linear time kernelization algorithm. We then prove its correctness and that the size of the produced instance can be bounded by $O(k^2)$. In Subsection \ref{subsec:FPTalgo}, we give an \FPT\ algorithm, which uses the kernelization algorithm, for checking the consistency of an \AMNV constraint in time $O(1.6181^k k^2 +|\mathcal{I}|)$. \HAC can then be enforced by applying this algorithm $O(|D|)$ times.
}

\subsection{Kernelization Algorithm}

Let $\mathcal{I}=(X,D,\dom,N)$ be an instance for the consistency problem for \AMNV constraints. 
\longversion{The algorithm is more intuitively described using the interval representation of the instance.}%
The \emph{friends} of an interval $I$ are the other intervals of $I$'s variable. An interval is \mydef{optional} if it has at least one friend, and \emph{required} otherwise. For a value $v\in D$, let $\ivl(v)$ denote the set of intervals containing $v$.

A \emph{solution} for $\mathcal{I}$ is a subset $S\subseteq D$ of at most $N$ values such that there exists an instantiation assigning the values in $S$ to the variables in $X$. 
The algorithm may detect for some value $v\in D$, that, if the problem has a solution, then it has a solution containing $v$. In this case, the algorithm \emph{selects} $v$, i.e., it removes all variables whose domain contains $v$, it removes $v$ from $D$, and it decrements $N$ by one.
The algorithm may detect for some value $v\in D$, that, if the problem has a solution, then it has a solution not containing $v$. In this case, the algorithm \emph{discards} $v$, i.e., it removes $v$ from every domain and from $D$. (Note that no new holes are created  with respect to $D\setminus \set{v}$.)
The algorithm may detect for some variable $x$, that every solution for $(X\setminus \set{x},D,\dom|_{X\setminus \set{x}},N)$ contains a value from $\dom(x)$. In that case, it \emph{removes} $x$. 

The algorithm sorts the intervals by increasing right endpoint (ties are broken arbitrarily). Then, it exhaustively applies the following three reduction rules.

\shortversion{\begin{description}[topsep=0pt, partopsep=0pt, itemsep=0pt]}
\longversion{\begin{description}}
\item{\RedIncl:} If there are two intervals $I,I'$ such that $I'\subseteq I$ and $I'$ is required, then remove the variable of $I$.
\item{\RedDom:} If there are two values $v,v'\in D$ such that $\ivl(v') \subseteq \ivl(v)$, then discard $v'$.
\item{\RedUnit:} If $|\dom(x)|=1$ for some variable $x$, then select the value in $\dom(x)$.
\end{description}


\noindent
In the example from Fig.~\ref{fig:input}, \RedIncl removes the variables $x_5$ and $x_8$ because $x_{10}\subseteq x_5'$ and $x_7\subseteq x_8$, \RedDom removes the values $1$ and $5$, \RedUnit selects $2$, which deletes variables $x_1$ and $x_2$, and \RedDom removes $3$ from $D$. The resulting instance is depicted in Fig.~\ref{fig:beforeScanning}.

After none of \longversion{the previous}\shortversion{these} rules apply, the algorithm scans the remaining intervals from left to right\longversion{ (i.e., by increasing right endpoint)}. An interval that has already been scanned is either a \mydef{leader} or a \mydef{follower} of a 
subset of leaders. Informally, for a leader $L$, if a solution contains $\rep(L)$, then there is a solution containing $\rep(L)$ and the right endpoint of each of its followers. 

\begin{figure} \centering
 \shortversion{\begin{tikzpicture}[xscale=0.6,yscale=0.65]}
 \longversion{\begin{tikzpicture}[xscale=0.65,yscale=0.75]}
  \tikzset{ivl/.style={thick},
  var/.style={midway,above=0.5pt}}
  \shortversion{\tikzset{ivlopt/.style={thick,black!60}}}
  \longversion{\tikzset{ivlopt/.style={thick,red}}}
 
  \draw (5,-0.6) node {$4$};
  \foreach \x in {6,...,14} \draw (\x,-0.6) node {$\x$};
  \foreach \x in {5,...,14} \draw[thin,black!20] (\x,-0.3) -- (\x,3.3);
  
  \draw (6.5,0.5) node[rectangle,draw] {$N=5$};
 
  \draw[ivlopt,|-] (4.9,3)-- node[var] {$x_3$} (5.1,3);
  \draw[ivl,|-|] (5.9,3)-- node[var] {$x_6$} (7.1,3);
  \draw[ivl,|-|] (7.9,3)-- node[var] {$x_9$} (9.1,3);
  \draw[ivlopt,-|] (9.9,3)-- node[var] {$x_3'$} (11.1,3);
  \draw[ivlopt,|-] (11.9,3)-- node[var] {$x_{13}$} (12.1,3);
  \draw[ivlopt,-|] (13.9,3)-- node[var] {$x_{13}'$} (14.1,3);
  \draw[ivl,|-|] (4.9,2)-- node[var] {$x_4$} (6.1,2);
  \draw[ivl,|-|] (6.9,2)-- node[var] {$x_7$} (8.1,2);
  \draw[ivl,|-|] (10.9,2)-- node[var] {$x_{12}$} (12.1,2);
  \draw[ivl,|-|] (12.9,2)-- node[var] {$x_{15}$} (14.1,2);
  \draw[ivl,|-|] (8.9,1)-- node[var] {$x_{10}$} (10.1,1);
  \draw[ivlopt,|-] (10.9,1)-- node[var] {$x_{11}$} (11.1,1);
  \draw[ivlopt,-|] (12.9,1)-- node[var] {$x_{11}'$} (13.1,1);
  \draw[ivl,|-|] (11.9,0)-- node[var] {$x_{14}$} (13.1,0);
 \end{tikzpicture}
 \vspace{-1pt}
 \caption{\label{fig:beforeScanning} Instance obtained from the instance of Fig.~\ref{fig:input} by exhaustively applying rules \RedIncl, \RedDom, and \RedUnit.}
\end{figure}
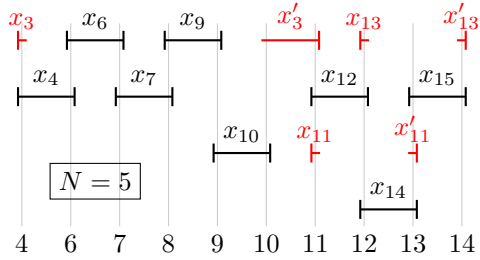

The algorithm scans the first intervals up to, and including, the first required interval. All these intervals become leaders.

The algorithm then continues scanning intervals one by one. Let $I$ be the interval that is currently scanned and $I_{p}$ be the last interval that was scanned. The \mydef{active} intervals are those that have already been scanned and intersect $I_{p}$. A \mydef{popular} leader is a leader that is either active or has at least one active follower.
\shortversion{\begin{itemize}[topsep=0pt, partopsep=0pt, itemsep=0pt]}
\longversion{\begin{itemize}}
 \item If $I$ is optional, then $I$ becomes a leader, the algorithm continues scanning intervals until scanning a required interval; all these intervals become leaders.
 \item If $I$ is required, then it becomes a follower of all popular leaders that do not intersect $I$ and that have no follower intersecting $I$. If all popular leaders have at least two followers, then set $N:=N-1$ and \textbf{merge} the second-last follower of each popular leader with the last follower of the corresponding leader; i.e., for every popular leader, the right endpoint of its second-last follower is set to the right endpoint of its last follower, and then the last follower of every popular leader is removed.
\end{itemize}

\noindent
After having scanned all the intervals, the algorithm exhaustively applies \longversion{the reduction rules}\shortversion{Rules} \RedIncl, \RedDom, and \RedUnit again.

In the example from Fig.~\ref{fig:beforeScanning}, variable $x_6$ is merged with $x_9$, and $x_7$ with $x_{10}$. \RedDom then removes the values $7$ and $8$, resulting in the instance depicted in Fig.~\ref{fig:kernel}.

\subsection{Correctness and Kernel Size}

Let $\mathcal{I}'=(X',D',\dom',N')$ be the instance resulting from applying one operation of the kernelization algorithm to an instance $\mathcal{I}=(X,D,\dom,N)$. An operation is an instruction which modifies the instance: \RedIncl, \RedDom, \RedUnit, and \textbf{merge}. 
We show that there exists a solution $S$ for $\mathcal{I}$ \myIff there exists a solution $S'$ for $\mathcal{I'}$. A solution is \emph{nice} if each of its elements is the right endpoint of some interval. Clearly, for every solution, a nice solution of the same size can be obtained by shifting each value to the next right endpoint of an interval. Thus, when we construct $S'$ from $S$ (or vice-versa), we may assume that $S$ is nice.

\longversion{Reduction }Rule \RedIncl is sound because a solution for $\mathcal{I}$ is a solution for $\mathcal{I'}$ and vice-versa, because any solution $\mathcal{I}'$ contains a value $v$ of $I\subseteq I'$, as $I$ is required.
\longversion{Reduction Rule }\RedDom is correct because if $v'\in S$, then $S':=(S\setminus \set{v'})\cup \set{v}$ is a solution for $\mathcal{I}'$ and for $\mathcal{I}$.
\longversion{Reduction Rule }\RedUnit is obviously correct ($S=S'\cup \dom(x)$).


After having applied these $3$ reduction rules, observe that the first interval is optional and contains only one value.
Suppose the algorithm has started scanning intervals.
By construction, the following properties apply to $\mathcal{I}'$.

\begin{figure} \centering
 \shortversion{\begin{tikzpicture}[xscale=0.6,yscale=0.65]}
 \longversion{\begin{tikzpicture}[xscale=0.65,yscale=0.75]}
  \tikzset{ivl/.style={thick},
  var/.style={midway,above=0.5pt}}
  \shortversion{\tikzset{ivlopt/.style={thick,black!60}}}
  \longversion{\tikzset{ivlopt/.style={thick,red}}}

  \draw (7,-0.6) node {$4$};
  \draw (8,-0.6) node {$6$};
  \foreach \x in {9,...,14} \draw (\x,-0.6) node {$\x$};
  \foreach \x in {7,...,14} \draw[thin,black!20] (\x,-0.3) -- (\x,3.3);
  
  \draw (8.5,0.5) node[rectangle,draw] {$N=4$};
 
  \draw[ivlopt,|-] (6.9,3)-- node[var] {$x_3$} (7.1,3);
  \draw[ivl,|-|] (7.9,3)-- node[var] {$x_6$} (9.1,3);
  \draw[ivlopt,-|] (9.9,3)-- node[var] {$x_3'$} (11.1,3);
  \draw[ivlopt,|-] (11.9,3)-- node[var] {$x_{13}$} (12.1,3);
  \draw[ivlopt,-|] (13.9,3)-- node[var] {$x_{13}'$} (14.1,3);
  \draw[ivl,|-|] (6.9,2)-- node[var] {$x_4$} (8.1,2);
  \draw[ivl,|-|] (8.9,2)-- node[var] {$x_7$} (10.1,2);
  \draw[ivl,|-|] (10.9,2)-- node[var] {$x_{12}$} (12.1,2);
  \draw[ivl,|-|] (12.9,2)-- node[var] {$x_{15}$} (14.1,2);
  \draw[ivlopt,|-] (10.9,1)-- node[var] {$x_{11}$} (11.1,1);
  \draw[ivlopt,-|] (12.9,1)-- node[var] {$x_{11}'$} (13.1,1);
  \draw[ivl,|-|] (11.9,0)-- node[var] {$x_{14}$} (13.1,0);
 \end{tikzpicture}
 \caption{\label{fig:kernel} Kernelized instance.}
\end{figure}
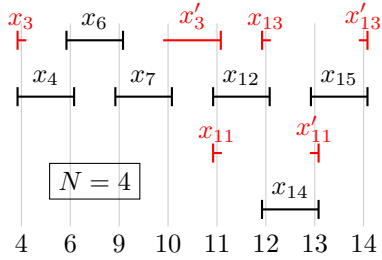

\begin{property}\label{obs:follower-not-intersect-leader}
 A follower does not intersect any of its leaders.
\end{property}

\begin{property}\label{obs:followers-no-intersect}
 If $I,I'$ are two (distinct) followers of the same leader, then $I$ and $I'$ do not intersect.
\end{property}

%

\noindent
Before proving the correctness of the \textbf{merge} operation, let us first show that the subset of leaders of a follower is not empty.
\begin{boldclaim}\label{cl:atleastoneleader}
 Every interval that has been scanned is either a leader or a follower of at least one leader.
\end{boldclaim}
\begin{proof}
First, note that \RedDom ensures that each \longversion{domain }value in $D$ is the left endpoint of some interval and the right endpoint of some interval.
Let $I$ be the interval that is currently scanned and $I_p$ be the previously scanned interval. If $I_p$ or $I$ is optional, then $I$ becomes a leader. Suppose $I$ and $I_p$ are required. We have \longversion{that }$\lep(I)>\lep(I_p)$, otherwise $I$ would have been removed by \RedIncl. By \longversion{Rule }\RedDom, there is some interval $I_\ell$ with $\rep(I_\ell)=\lep(I_p)$. If $I_\ell$ is a leader, $I$ becomes a follower of $I_\ell$; otherwise $I$ becomes a follower of $I_\ell$'s leader.
\end{proof}

\noindent
The following two lemmas prove the correctness of the \textbf{merge} operation.
Recall that $\mathcal{I}'$ is an instance obtained from $\mathcal{I}$ by one application of the \textbf{merge} operation.

\begin{lemma}
If $S$ is a nice solution for $\mathcal{I}$, then there exists a solution $S'$ for $\mathcal{I}'$ with $S'\subseteq S$.
\end{lemma}
\begin{proof}
Consider the step where the kernelization algorithm applies the \textbf{merge} operation. At that step, each popular leader has at least two followers and the algorithm merges the last two followers of each popular leader and decrements $N$ by one. The currently scanned interval is $I$. Let $F_2$ denote the set of \longversion{all }intervals that are the second-last follower of a popular leader, and $F_1$ the set of \longversion{all }intervals that are the last follower of a popular leader before merging. Let $M$ denote the set of merged intervals. Clearly, every interval of $F_1\cup F_2\cup M$ is required as all followers are required.

\begin{boldclaim}\label{cl:lastfollower-commonvalue}
Every interval in $F_1$ intersects $\lep(I)$.
\end{boldclaim}
\begin{proof}
Let $I_1\in F_1$.
By construction, $\rep(I_1)\in I$, as $I$ becomes a follower of every popular leader that has no follower intersecting $I$, and no follower has a right endpoint larger than $\rep(I)$. Moreover, $\lep(I_1)\le \lep(I)$ as no follower is a strict subset of $I$ by \RedIncl and the fact that all followers are required.
\end{proof}

\noindent
Let $I^-$ be the interval of $F_2$ with the largest right endpoint. Let $L$ be a leader of $I^-$. By construction and \RedIncl, $L$ is a leader of $I$\longversion{ as well} and is thus popular. Let $t_1\in S\cap I$ be the smallest value of $S$ that intersects $I$ and let $t_2 \in S \cap I^-$ be the largest value of $S$ that intersects $I^-$. By Property \ref{obs:followers-no-intersect}, $t_2<t_1$.

\begin{boldclaim}\label{cl:no-intermediate-value}
 $S$ contains no value $t_0$ such that $t_2<t_0<t_1$.
\end{boldclaim}
\begin{proof}
Suppose $S$ contained such a value $t_0$. As $S$ is nice, $t_0$ is the right endpoint of some interval $I_0$. As $t_2$ is the rightmost value intersecting $S$ and any interval in $F_2$, $I_0$ is not in $F_2$. As $I_0$ has already been scanned, and was scanned after every interval in $F_2$, $I_0$ is in $F_1$. However, by Claim \ref{cl:lastfollower-commonvalue}, $I_0$ intersects $\lep(I)$. As no scanned interval has a larger right endpoint than $I$, $t_0 \in S\cap I$, which contradicts the fact that $t_1$ is the smallest value in $S\cap I$ and that $t_0<t_1$.
\end{proof}

\begin{boldclaim}\label{cl:t1intersectsmerged}
 Suppose $I_1\in F_1$ and $I_2\in F_2$ are the last and second-last follower of a popular leader $L'$, respectively. Let $M_{12}\in M$ denote the interval obtained from merging $I_2$ with $I_1$. If $t_2\in I_2$, then $t_1\in M_{12}$.
\end{boldclaim}
\begin{proof}
\longversion{For the sake of contradiction, assume}\shortversion{Assume otherwise that} $t_2 \in I_2$, but $t_1 \notin M_{12}$. As $t_2 < t_1$, we have\longversion{ that} $t_1 > \rep(M_{12}) = \rep(I_1)$. But then $S$ is not a solution as \shortversion{$t_2<\lep(I_1)$ and }$S\cap I_1 = \emptyset$ by Claim \ref{cl:no-intermediate-value}\longversion{ and the fact that $t_2<\lep(I_1)$}.
\end{proof}

\begin{boldclaim}\label{cl:inter-t2}
If $I'$ is an interval with $t_2\in I'$, then $I'\in F_2\cup F_1$.
\end{boldclaim}
\begin{proof}
First, suppose $I'$ is a leader. As every leader has at least two followers when $I$ is scanned, $I'$ has two followers whose left endpoint is larger than $\rep(I')\ge t_2$ (by Property \ref{obs:follower-not-intersect-leader}) and smaller than $\lep(I)\le t_1$ (by \RedIncl). Thus, at least one of them is included in the interval $(t_2,t_1)$ by Property \ref{obs:followers-no-intersect}, which contradicts $S$ being a solution by Claim \ref{cl:no-intermediate-value}.

Similarly, if $I'$ is a follower of a popular leader, but not among the last two followers of any popular leader, Claim \ref{cl:no-intermediate-value} leads to a contradiction as well.

Finally, if $I'$ is a follower, but has no popular leader, then it is to the left of \longversion{some}\shortversion{a} popular leader, and thus to the left of $t_2$.
\end{proof}

\noindent
Consider the set $T_2$ of intervals that intersect $t_2$. By Claim \ref{cl:inter-t2}, $T_2 \subseteq F_2\cup F_1$. For every interval $I'\in T_2\cap F_2$, the corresponding merged interval of $\mathcal{I}'$ intersects $t_1$ by Claim \ref{cl:t1intersectsmerged}. For every interval $I'\in T_2\cap F_1$, and every interval $I''\in F_2$ with which $I'$ is merged, $S$ contains some value $x\in I''$ with $x<t_2$. Thus, $S':=S\setminus \{t_2\}$ is a solution for $\mathcal{I}'$.
\end{proof}

\begin{lemma}
If $S'$ is a nice solution for $\mathcal{I}'$, then there exists a solution $S$ for $\mathcal{I}$ with $S'\subseteq S$.
\end{lemma}
\begin{proof}
As in the previous proof, consider the step where the kernelization algorithm applies the \textbf{merge} operation. The currently scanned interval is $I$. Let $F_2$ and $F_1$ denote the set of \longversion{all }intervals that are the second-last and last follower of a popular leader before merging, respectively. Let $M$ denote the set of merged intervals. 

By Claim \ref{cl:lastfollower-commonvalue}\longversion{ from the previous proof}, every interval of $M$ intersects $\lep(I)$. On the other hand, every interval of $\mathcal{I}'$ whose right endpoint intersects $I$ is in $M$, by construction. Thus, $S'$ contains the right endpoint of some interval of $M$. Let $t_1$ denote the smallest such value, and let $I_1$ denote the interval of $\mathcal{I}$ with $\rep(I_1)=t_1$ (due to \RedIncl, there is a unique such interval). Let $I_2$ denote the interval of $\mathcal{I}$ with the smallest right endpoint such that there is a leader $L$ whose second-last follower is $I_2$ and whose last follower is $I_1$, and let $t_2:=\rep(I_2)$.

\begin{boldclaim}
Let $I_1'\in F_1$ and $I_2'\in F_2$ be two intervals \longversion{from $\mathcal{I}$ }that are merged into one interval $M_{12}'$\longversion{ of $\mathcal{I}'$}. If $t_1\in M_{12}'$, then $t_2\in I_2'$.
\end{boldclaim}
\begin{proof}
Suppose $t_1 \in M_{12}'$ but $t_2 \notin I_2'$. 
\longversion{We consider two cases. In the fist case,}\shortversion{If} $I_2' \subseteq (t_2,\lep(I_1'))$\longversion{. But then,}\shortversion{, then} $I_2'$ would have become a follower of $L$, which contradicts that $I_1$ is the last follower of $L$. \longversion{In the second case}\shortversion{Otherwise}, $\rep(I_2')<t_2$. But then, $I_1$ is a follower of the same leader as $I_1'$, as $\lep(I_1)\le \lep(I_1')$, and thus $I_1=I_1'$. By definition of $I_2$, however, $t_2=\rep(I_2)\le \rep(I_2')$, a contradiction.
\end{proof}

\noindent
By the previous claim, a solution $S$ for $\mathcal{I}$ is obtained from a solution $S'$ for $\mathcal{I}'$ by setting $S:=S'\cup \{t_2\}$.
\end{proof}

\noindent
\longversion{After having scanned all the intervals, Reduction Rules \RedIncl, \RedDom, and \RedUnit are applied again, and we have already proved their correctness.

\smallskip
}%
Thus, the kernelization algorithm returns an equivalent instance. To bound the kernel size by a polynomial in $k$, let $\mathcal{I}^*=(V^*,D^*,\dom^*,N^*)$ be the instance resulting from applying the \longversion{kernelization }algorithm to an instance $\mathcal{I}=(V,D,\dom,N)$.

\begin{property}\label{obs:nboptint}
 $\mathcal{I}$ and $\mathcal{I}^*$ have at most $2k$ optional intervals.
\end{property}

\noindent
Property \ref{obs:nboptint} holds for $\mathcal{I}$ as every optional interval is adjacent to at least one hole and each hole is adjacent to two optional intervals. It holds for $\mathcal{I}^*$ as the \longversion{kernelization }algorithm introduces no holes.

\begin{lemma}\label{lem:nbleaders}
 $\mathcal{I}^*$ has at most $4k$ leaders.
\end{lemma}
\begin{proof}
\shortversion{For every required interval that becomes a leader, an optional interval also becomes a leader. As every interval is scanned only once, the number of leaders is at most $4k$ by Property \ref{obs:nboptint}.}%
\longversion{Consider the unique step of the algorithm that creates leaders. An optional interval is scanned, the algorithm continues scanning intervals until scanning a required interval, and all these scanned intervals become leaders. As every interval is scanned only once, for every optional interval, there are at most $2$ leaders. By Property \ref{obs:nboptint}, the number of leaders is thus at most $4k$.}
\end{proof}

\shortversion{The following lemma can be proved by analyzing how many new followers a leader can get in a period where no optional interval is scanned. The analysis considers different
cases based on the popular leader with the rightmost right endpoint, and is omitted due to space constraints.}

\begin{lemma}\label{lem:nbfollowers}
 Every leader has at most $4k$ followers.
\end{lemma}
\longversion{\begin{proof}
Consider all steps where a newly scanned interval becomes a follower, but is not merged\longversion{ with another interval}. In each of these steps, the popular leader $L_r$ with the rightmost right endpoint either
\begin{itemize}
\item[(a)] has no follower and intersects $I$, or
\item[(b)] has no follower and does not intersect $I$, or
\item[(c)] has one follower and intersects $I$.
\end{itemize}
Now, let $L$ be some leader and let us consider a period where no optional interval is scanned. Let us bound the number of intervals that become followers of $L$ during this period without being merged with another interval. If the number of followers of $L$ increases in Situation (a), it does not increase in Situation (a) again during this period, as no other follower of $L$ may intersect $I$. After Situation (b) occurs, Situation (b) does not occur again during this period, as $I$ becomes a follower of $L_r$. Moreover, the number of followers of $L$ does not increase during this period in Situation (c) after Situation (b) has occurred, as no other follower of $L$ may intersect $I$. After Situation (c) occurs, the number of followers of $L$ does not increase in Situation (c) again during this period, as no other follower of $L$ may intersect $I$. Thus, at most $2$ followers are added to $L$ in each period. As the first scanned interval is optional, Property \ref{obs:nboptint} bounds the number of periods by $2k$. Thus, $L$ has at most $4k$ followers.
\end{proof}}

\noindent
\longversion{As, by }\shortversion{By }Claim \ref{cl:atleastoneleader}\longversion{, every interval of $\mathcal{I}^*$ is either a leader or a follower of at least one leader,}\shortversion{ and} Lemmas \ref{lem:nbleaders} and \ref{lem:nbfollowers}\longversion{ imply that}\shortversion{,} $\mathcal{I}^*$ has $O(k^2)$ intervals\longversion{, and thus $|X^*|=O(k^2)$}. \longversion{Because of Reduction Rule}\shortversion{By} \RedDom, every value in $D^*$ is the right endpoint and the left endpoint of some interval\longversion{, and thus, $|D^*|=O(k^2)$}.\shortversion{ Thus, $|X^*|+|D^*|=O(k^2)$.}
\longversion{
\smallskip

Using a counting sort algorithm with satellite data (see, e.g., \cite{CormenLeisersonRivestStein09}), the initial sorting of the $n+k$ intervals can be done in time $O(n+|D|+k)$. To facilitate the application of \RedIncl, counting sort is actually used twice to also sort by increasing left endpoint the sets of intervals with coinciding right endpoint.
An optimized implementation applies \RedIncl, \RedDom and \RedUnit simultaneously in one pass through the intervals, as one rule might trigger the other. To guarantee a linear running time for the scan-and-merge phase of the algorithm, only the first follower of a leader stores a pointer to the leader; all other followers store a pointer to the previous follower.
Due to space limitations, we omit the formal details about the implementation and running time analysis of the kernelization algorithm.
}We \shortversion{omit the running time analysis and} arrive at our main theorem.

\begin{theorem}\label{thm:kernel}
  The Consistency problem for \AMNV constraints,
  parameterized by the number $k$ of holes, admits a linear time
  reduction  to a problem kernel with $O(k^2)$ variables and  $O(k^2)$ domain
  values.
\end{theorem}

\noindent
Using the succinct description of the domains, the size of the kernel can be bounded by $O(k^2)$.

\longversion{\smallskip

\noindent
\emph{Remark:} Denoting $\textsf{var}(v)=\set{x\in X: v\in \dom(x)}$, Rule \RedDom can be generalized to discard any $v'\in D$ for which there exists a $v\in D$ such that $\textsf{var}(v')\subseteq \textsf{var}(v)$ at the expense of a higher running time.
}

\subsection{Improved FPT Algorithm and \HAC}
\label{subsec:FPTalgo}

Using the kernel from Theorem \ref{thm:kernel} and the simple algorithm described in the beginning of this section, one arrives at a $O(2^k k^2 + |\mathcal{I}|)$ time algorithm for checking the consistency of an \AMNV constraint. Borrowing ideas from the kernelization algorithm, we now reduce the exponential dependency on $k$ in the running time.
\longversion{The speed-ups due to this branching algorithm and the kernelization algorithm lead to a speed-up for enforcing \HAC for \AMNV constraints (by Corollary \ref{cor:HAC}) and for enforcing \HAC for \NV constraints (by the decomposition of \cite{BessiereEtal06}).}

\begin{theorem}\label{thm:FPTalgo}
The Consistency problem for \AMNV constraints admits a $O(\rho^k k^2+|\mathcal{I}|)$ time algorithm, where $k$ is the number of holes in the domains of the input instance $\mathcal{I}$, and $\rho=\frac{1+\sqrt{5}}{2}<1.6181$.
\end{theorem}
\begin{proof}
The first step of the algorithm invokes the kernelization algorithm and obtains an equivalent instance $\mathcal{I}'$ with $O(k^2)$ intervals in time $O(|\mathcal{I}|)$.

Now, we describe a branching algorithm checking the consistency of $\mathcal{I}'$. Let $I_1$ denote the first interval of $\mathcal{I}'$ (in the ordering by increasing right endpoint). $I_1$ is optional. Let $\mathcal{I}_1$ denote the instance obtained from $\mathcal{I}'$ by selecting $\rep(I_1)$ and exhaustively applying \longversion{Reduction }Rules \RedDom and \RedUnit. Let $\mathcal{I}_2$ denote the instance obtained from $\mathcal{I}'$ by removing $I_1$ (if $I_1$ had exactly one friend, this friend becomes required) and exhaustively applying \longversion{Reduction Rules }\RedDom and \RedUnit. Clearly, $\mathcal{I}'$ is consistent \myIff $\mathcal{I}_1$ or $\mathcal{I}_2$ is consistent.

Note that both $\mathcal{I}_1$ and $\mathcal{I}_2$ have at most $k-1$ holes. If either $\mathcal{I}_1$ or $\mathcal{I}_2$ has at most $k-2$ holes, the algorithm recursively checks whether at least one of $\mathcal{I}_1$ and $\mathcal{I}_2$ is consistent.
If both $\mathcal{I}_1$ and $\mathcal{I}_2$ have exactly $k-1$ holes, we note that in $\mathcal{I}'$,
\shortversion{\begin{itemize}[topsep=0pt, partopsep=0pt, itemsep=0pt]}%
\longversion{\begin{itemize}}
\item[(1)] $I_1$ has one friend,
\item[(2)] no other optional interval intersects $I_1$, and
\item[(3)] the first interval of both $\mathcal{I}_1$ and $\mathcal{I}_2$ is $I_f$, which is the third optional interval in $\mathcal{I}'$ if the second optional interval is the friend of $I_1$, and the second optional interval otherwise.
\end{itemize}
Thus, the instance obtained from $\mathcal{I}_1$ by removing $I_1$'s friend and applying \RedDom and \RedUnit may differ from $\mathcal{I}_2$ only in $N$. Let $s_1$ and $s_2$ denote the number of values smaller than $\rep(I_f)$ that have been selected to obtain $\mathcal{I}_1$ and $\mathcal{I}_2$ from $\mathcal{I}'$, respectively. If $s_1 \le s_2$, then the non-consistency of $\mathcal{I}_1$ implies the non-consistency of $\mathcal{I}_2$. Thus, the algorithm need only recursively check whether $\mathcal{I}_1$ is consistent. On the other hand, if $s_1 > s_2$, then the non-consistency of $\mathcal{I}_2$ implies the non-consistency of $\mathcal{I}_1$. Thus, the algorithm need only recursively check whether $\mathcal{I}_2$ is consistent.

The recursive calls of the algorithm may be represented by a search tree labeled with the number of holes of the instance. As the algorithm either branches into only one subproblem with at most $k-1$ holes, or two subproblems with at most $k-1$ and at most $k-2$ holes, respectively, the number of leaves of this search tree is
$
T(k) \le T(k-1)+T(k-2),
$
with $T(0)=T(1)=1$. Using standard techniques in the analysis of exponential time algorithms (see, e.g., \cite{FominKratsch10}), and by noticing that the number of operations executed at each node of the search tree is $O(k^2)$, the running time of the \longversion{branching }algorithm can be upper bounded by $O(\rho^k k^2)$.
\end{proof}

\noindent
For the example of Fig.~\ref{fig:kernel}, the instances $\mathcal{I}_1$ and $\mathcal{I}_2$ are computed by selecting the value $4$, and removing the interval $x_3$, respectively. The reduction rules select the value $9$ for $\mathcal{I}_1$ and the values $6$ and $10$ for $\mathcal{I}_2$. Both instances start with the interval $x_{11}$, and the algorithm recursively solves $\mathcal{I}_1$ only, where \longversion{the values $12$ and $13$ are selected, leading to} the solution $\set{4,9,12,13}$ \shortversion{is obtained} for the kernelized instance, which corresponds to the solution $\set{2,4,7,9,12,13}$ for the instance of Fig.~\ref{fig:input}.

\begin{corollary}\label{cor:HAC}
\HAC for an \AMNV constraint can be enforced in time $O(\rho^k \cdot k^2 \cdot |D|+|\mathcal{I}|\cdot |D|)$, where $k$ is the number of holes in the domains of the input instance $\mathcal{I}=(X,D,\dom,N)$, and $\rho=\frac{1+\sqrt{5}}{2}<1.6181$.
\end{corollary}
\begin{proof}
We first remark that if a value $v$ can be filtered from the domain of a variable $x$ (i.e., $v$ has no support for $x$), then $v$ can be filtered from the domain of all variables, as for any legal instantiation $\alpha$ with $\alpha(x')=v$, $x'\in X\setminus \set{x}$, the assignment obtained from $\alpha$ by setting $\alpha(x):=v$ is a legal instantiation as well. 
Also, filtering the value $v$ creates no new holes as the set of values can be set to $D\setminus \set{v}$.

Now we enforce \HAC by applying $O(|D|)$ times
the algorithm from Theorem \ref{thm:FPTalgo}.
Assume the instance $\mathcal{I}=(X,D,\dom,N)$ is consistent. If $(X,D,\dom,N-1)$ is consistent, then no value can be filtered.
Otherwise, check, for each $v\in D$, whether the instance obtained from selecting $v$ is consistent and filter $v$ if this is not the case.
\end{proof}

\section{Extended Global Cardinality Constraints}

An \EGC constraint $C$ is specified
by a set of variables $\longversion{\scope(C)=}\{x_1,\dots,x_n\}$ and for each value
$v\in \bigcup_{i=1}^n \dom(x_i)$ a set $D(v)$ of non-negative
integers. \longversion{The constraint}\shortversion{$C$} is consistent if each
variable can take a value from its domain such that the number of variables
taking a value $v$ belongs to the set $D(v)$.

The Consistency problem for \EGC constraints is 
$\NP$-hard \cite{QuimperLopezortizVanbeekGolynski04}\longversion{. However, if all
sets $D(\cdot)$ are intervals, then consistency can be checked in polynomial
time using network flows~\cite{Regin96}.}\shortversion{ in general and polynomial time solvable
if all sets $D(\cdot)$ are intervals~\cite{Regin96}.} By the result of \citex{BessiereEtal08},
the Consistency problem for \EGC constraints is \longversion{fixed-parameter tractable}\shortversion{FPT},
parameterized by the number of holes in the sets $D(\cdot)$. Thus R\'{e}gin's
result generalizes to instances that are close to the interval case.
\longversion{ }%
However, it is unlikely that \EGC constraints admit a polynomial kernel.
\begin{theorem}
  The Consistency problem for \EGC constraints, parameterized by the
  number of holes in the sets $D(\cdot)$, does not admit a polynomial kernel unless
  $\NP \subseteq \coNP/\text{\normalfont poly}$.
\end{theorem}
\begin{proof}
  \longversion{We establish the theorem by a combination of results from
  \citex{BodlaenderThomasseYeo09}, \citex{FortnowSanthanam08}, and
  \citex{QuimperLopezortizVanbeekGolynski04}. We need the following
  definitions. }The \emph{unparameterized version} of a parameterized
  problem $\PP\subseteq \Sigma^* \times \Nat$ is
  $\UP(\PP)=\SB x\#1^k \SM
  (x,k)\in P \SE\subseteq (\Sigma \cup \{ \#\})^*$ where $1$ is
  an arbitrary symbol from $\Sigma$ and $\#$ is a new symbol not in
  $\Sigma$. 
\longversion{ }%
  Let $\PP,\QQ\subseteq \Sigma^* \times \Nat$ be parameterized
  problems. \longversion{We say that }$\PP$ is \emph{polynomial parameter reducible}
  to $\QQ$ if there is a polynomial time computable function
  $f: \Sigma^* \times \Nat \rightarrow \Sigma^* \times \Nat$ and a
  polynomial $p$, such that for all $(x,k) \in \Sigma^* \times \Nat$, we have
  $(x,k) \in \PP$ \myIff $(x',k') = f(x,k) \in \QQ$,
  and $k'\leq p(k)$. 

  We prove the theorem by combining three known
  results.
  \shortversion{\begin{enumerate}[topsep=0pt, partopsep=0pt, itemsep=0pt]}
  \longversion{\begin{enumerate}}
  \item[(1)] \cite{BodlaenderThomasseYeo09} Let $\PP$ and $\QQ$ be
    parameterized problems such that $\UP(\PP)$ is $\NP$-complete,
    $\UP(\QQ)$ is in $\NP$, and $\PP$ is polynomial parameter
    reducible to~$\QQ$. If $\QQ$ has a polynomial
    kernel, then $\PP$ has a polynomial kernel.
  \item[(2)] \cite{FortnowSanthanam08} The problem of deciding the
    satisfiability of a CNF formula (SAT), parameterized by the number
    of variables, does not admit a polynomial kernel, unless
    $\NP \subseteq \coNP/\text{\normalfont poly}$.
  \item[(3)] \cite{QuimperLopezortizVanbeekGolynski04}
    Given a CNF formula $F$ on $k$ variables, one can construct in polynomial
    time an \EGC constraint $C_F$ such that
    \shortversion{\begin{enumerate}[topsep=0pt, partopsep=0pt, itemsep=0pt]}
    \longversion{\begin{enumerate}}
     \item[(i)] for each value $v$ of $C_F$, $D(v)=\{0,i_v\}$ for an integer $i_v>0$,
     \item[(ii)] $i_v>1$ for at most $2k$ values $v$, and
     \item[(iii)] $F$ is satisfiable \myIff $C_F$ is consistent.
    \end{enumerate}
    Thus, the number of holes in $C_F$ is at most twice the number of
    variables of $F$.
  \end{enumerate}
  We observe that (3) is a polynomial parameter reduction from SAT,
  parameterized by the number of variables, to the Consistency problem
  for \EGC\longversion{ constraints}, parameterized by the number of holes. Hence the
  theorem follows from (1) and (2).
  \end{proof}

\section{Conclusion}

We have introduced the concept of kernelization to the field of
constraint processing, providing both positive and negative results for
the important global constraints \NV and \EGC, respectively. On
the positive side, we have developed an efficient linear-time
kernelization algorithm for the consistency problem for
\AMNV constraints, and have shown how it can be used
to speed up the complete propagation of \NV and related
constraints. On the negative side, we have established a theoretical
result which indicates that \EGC constraints do not admit
polynomial kernels.

Our algorithms are efficient and the theoretical worst-case time bounds
do not include large hidden constants. We therefore believe that the
algorithms are practical, but we must leave an empirical evaluation for
future research. We hope that our results stimulate further
research on kernelization algorithms for constraint processing. 

 
\bibliographystyle{named}

\begin{thebibliography}{}\shortversion{\setlength{\itemsep}{-2pt}}

\bibitem[\protect\citeauthoryear{Beldiceanu}{2001}]{Beldiceanu01}
N.~Beldiceanu.
\newblock Pruning for the minimum constraint family and for the number of
  distinct values constraint family.
\newblock In {\em CP 01}, pp. 211--224, 2001.

\bibitem[\protect\citeauthoryear{Bessi{\`e}re \bgroup \em et al.\egroup
  }{2004}]{BessiereEtAl04}
C.~Bessi{\`e}re, E.~Hebrard, B.~Hnich, and T.~Walsh.
\newblock The complexity of global constraints.
\newblock In {\em IAAI 04}, pp. 112--117, 2004.

\bibitem[\protect\citeauthoryear{Bessi{\`e}re \bgroup \em et al.\egroup
  }{2006}]{BessiereEtal06}
C.~Bessi{\`e}re, E.~Hebrard, B.~Hnich, Z.~Kiziltan, and T.~Walsh.
\newblock Filtering algorithms for the {NV}alue constraint.
\newblock {\em Constraints}, 11(4):271--293, 2006.

\bibitem[\protect\citeauthoryear{Bessi{\`e}re \bgroup \em et al.\egroup
  }{2008}]{BessiereEtal08}
C.~Bessi{\`e}re, E.~Hebrard, B.~Hnich, Z.~Kiziltan, C.-G. Quimper, and
  T.~Walsh.
\newblock The parameterized complexity of global constraints.
\newblock In {\em AAAI 08}, pp. 235--240, 2008.

\bibitem[\protect\citeauthoryear{Bessi{\`e}re}{2006}]{Bessiere06}
C.~Bessi{\`e}re.
\newblock Constraint propagation.
\newblock In {\em Handbook of Constraint Programming}, chapter~3. Elsevier,
  2006.

\bibitem[\protect\citeauthoryear{Bodlaender \bgroup \em et al.\egroup
  }{2009a}]{BodlaenderDowneyFellowsHermelin09}
H.~L. Bodlaender, R.~G. Downey, M.~R. Fellows, and D.~Hermelin.
\newblock On problems without polynomial kernels.
\newblock {\em J. Comput. Syst. Sci.}, 75(8):423--434, 2009.

\bibitem[\protect\citeauthoryear{Bodlaender \bgroup \em et al.\egroup
  }{2009b}]{BodlaenderThomasseYeo09}
H.~L. Bodlaender, S.~Thomass{\'e}, and A.~Yeo.
\newblock Kernel bounds for disjoint cycles and disjoint paths.
\newblock In {\em ESA 09}, Springer {\em LNCS} 5757, pp. 635--646,
  2009.

\longversion{
\bibitem[\protect\citeauthoryear{Cormen \bgroup \em et al.\egroup}{2009}]{CormenLeisersonRivestStein09}
T.~H. Cormen, C.~E. Leiserson, R.~L. Rivest, and C. Stein.
\newblock {\em Introduction to Algorithms, Third Edition}.
\newblock The MIT Press, 2009.}%


\bibitem[\protect\citeauthoryear{Downey and Fellows}{1999}]{DowneyFellows99}
R.~G. Downey and M.~R. Fellows.
\newblock {\em Parameterized Complexity}.
\newblock Springer, 1999.

\bibitem[\protect\citeauthoryear{Fellows}{2006}]{Fellows06}
M.~R. Fellows.
\newblock The lost continent of polynomial time: Preprocessing and
  kernelization.
\newblock In {\em IWPEC 06}, Springer {\em LNCS}  4169, pp. 276--277,
  2006.

\bibitem[\protect\citeauthoryear{Fomin and Kratsch}{2010}]{FominKratsch10}
F.~V. Fomin and D.~Kratsch.
\newblock {\em Exact Exponential Algorithms}.
\newblock Springer, 2010.

\bibitem[\protect\citeauthoryear{Fomin}{2010}]{Fomin10}
F.~V. Fomin.
\newblock Kernelization.
\newblock In {\em CSR 10}, Springer {\em LNCS} 6072, pp. 107--108,
  2010.

\bibitem[\protect\citeauthoryear{Fortnow and
  Santhanam}{2008}]{FortnowSanthanam08}
L.~Fortnow and R.~Santhanam.
\newblock Infeasibility of instance compression and succinct {PCPs} for {NP}.
\newblock In {\em STOC 08}, pp. 133--142, 2008.

\bibitem[\protect\citeauthoryear{Guo and Niedermeier}{2007}]{GuoNiedermeier07}
J.~Guo and R.~Niedermeier.
\newblock Invitation to data reduction and problem kernelization.
\newblock {\em ACM SIGACT News}, 38(2):31--45, March 2007.

\bibitem[\protect\citeauthoryear{Pachet and Roy}{1999}]{PachetRoy99}
F.~Pachet and P.~Roy.
\newblock Automatic generation of music programs.
\newblock In {\em CP 99}, pp. 331--345. Springer, 1999.

\bibitem[\protect\citeauthoryear{Quimper \bgroup \em et al.\egroup
  }{2004}]{QuimperLopezortizVanbeekGolynski04}
C.-G. Quimper, A.~L\'{o}pez-Ortiz, P.~van Beek, and A.~Golynski.
\newblock Improved algorithms for the global cardinality constraint.
\newblock In {\em CP 04}, \longversion{Springer {\em LNCS} 3258, }pp. 542--556,
  2004.

\bibitem[\protect\citeauthoryear{R\'{e}gin}{1996}]{Regin96}
J.-C. R\'{e}gin.
\newblock Generalized arc consistency for global cardinality constraint.
\newblock In {\em AAAI 96}, vol.~1, pp. 209--215, 1996.

\bibitem[\protect\citeauthoryear{Rosamond}{2010}]{Rosamond10}
F.~Rosamond.
\newblock Table of races.
\newblock In {\em Parameterized Complexity Newsletter}, pp. 4--5. 2010.
\newblock \url{http://fpt.wikidot.com/}.

\bibitem[\protect\citeauthoryear{Rossi \bgroup \em et al.\egroup
  }{2006}]{RossiVanBeekWalsh06}
F.~Rossi, P.~van Beek, and T.~Walsh, editors.
\newblock {\em Handbook of Constraint Programming}.
\newblock Elsevier, 2006.

\bibitem[\protect\citeauthoryear{Smith}{2006}]{Smith06}
B.~M. Smith.
\newblock Modelling.
\newblock In {\em Handbook of Constraint Programming}, chapter~11. Elsevier,
  2006.

\bibitem[\protect\citeauthoryear{van Hoeve and Katriel}{2006}]{HoeveKatriel06}
W.-J. van Hoeve and I.~Katriel.
\newblock Global constraints.
\newblock In {\em Handbook of Constraint Programming}, chapter~6. Elsevier,
  2006.
\end{thebibliography}
{\shortversion{\footnotesize}

}
    
\end{document}